\documentclass[conference]{IEEEtran}

\usepackage{amsmath,amssymb,amsfonts}
\usepackage{graphicx}
\usepackage{textcomp}
\usepackage[dvipsnames]{xcolor}
\usepackage{multirow}
\usepackage{booktabs}
\usepackage{siunitx}
\usepackage[ruled,vlined]{algorithm2e}
\usepackage{pgfplots}
\pgfplotsset{width=1\linewidth,compat=1.9}
\pgfplotsset{compat=newest}
\usepackage{tikz}
\usepackage{array}
\usepackage{stfloats}
\usepackage{url}
\usepackage{verbatim}

\usetikzlibrary{positioning,calc}

\usepackage{subcaption}
\usepackage{amsthm}

\usepackage[caption=false,font=footnotesize,labelfont=sf,textfont=sf]{subfig}

\newtheorem{theorem}{Theorem}

\newcommand{\ours}{\textsc{DART-FL}}

\title{DART-FL: Burst-Aware Multitask Federated Learning under Dynamic Inference Demand at the Edge}

\author{
\IEEEauthorblockN{
Yiming Xie\textsuperscript{1},
Pinrui Yu\textsuperscript{1},
Geng Yuan\textsuperscript{2},
Xue Lin\textsuperscript{1},
Ningfang Mi\textsuperscript{1}
}
\IEEEauthorblockA{
\textsuperscript{1}\textit{Department of Electrical and Computer Engineering, Northeastern University, Boston, MA, USA}\\
\textsuperscript{2}\textit{School of Computing, University of Georgia, Athens, GA, USA}\\
Email:
\textsuperscript{1}\{xie.yimi, yu.pin, xue.lin, n.mi\}@northeastern.edu,
\textsuperscript{2}geng.yuan@uga.edu
}
}

\begin{document}

\maketitle



\begin{abstract}
Edge intelligence systems increasingly require model training and online
inference to coexist on resource-constrained devices, while inference demand
can vary substantially across tasks over time. This creates two coupled
challenges: sufficient computation must be reserved for inference to maintain
service-level objectives (SLOs), while the remaining training capacity should
adapt to task-specific demand so that frequently requested tasks can improve
earlier during training.

We propose an SLO-aware, demand-driven multitask federated learning framework (\ours{}) that jointly adapts the inference--training resource split and task-level
training emphasis. At each scheduling interval, \ours{} uses the
inference backlog and profiled service capacity to determine the minimum
resource allocation required for inference. The remaining training capacity
is then distributed across tasks using a queue-aware DPP-inspired scheduler, and the resulting task allocations are mapped to dynamic loss weights. This allows
tasks experiencing higher inference demand to receive greater training
emphasis in earlier communication rounds. Clients train a shared backbone
with task-specific heads, and the complete multitask model is aggregated
through FedAvg.

We evaluate \ours{} using Stanford Cars and Oxford Flowers 102 under
both synthetic and real Alibaba trace-derived workloads.
Results show that \ours{} dynamically adapts the inference-training resource split to
time-varying inference demand and shifts the learning progress of high-demand
tasks toward their burst periods, improving model accuracy when those tasks
are frequently requested while maintaining comparable long-term multitask
performance.
\end{abstract}
\section{Introduction}

Edge devices are increasingly required to support both continuous model
training and real-time inference services. Applications such as autonomous
driving, mobile vision assistants, and on-device language understanding
rely on models that must continuously adapt to newly collected data while
simultaneously serving user requests~\cite{li2024edgeol,han2024joint}.
In such systems, computation resources are shared between training and
inference, creating a fundamental tension: allocating more resources to
inference reduces the capacity available for model improvement, whereas
aggressive training may increase inference latency and lead to
service-level objective (SLO) violations~\cite{han2024joint}.

This challenge becomes more significant in \emph{multitask} edge
applications. For example, an autonomous vehicle may simultaneously support
object recognition, scene understanding, and traffic-related perception
tasks. Similarly, mobile and embedded devices increasingly execute multiple
vision or language services using a shared backbone with task-specific
prediction heads. Such architectures enable parameter sharing and joint
learning across related tasks, while the relative weighting of task
objectives can substantially affect multitask optimization
~\cite{kendall2018multitask,sener2018multiobjective}.

In practice, inference demand can also be highly \emph{bursty}~\cite{li2023alpaserve}. A sudden
increase in requests for one task can rapidly build up its inference queue,
requiring additional computation resources to maintain the target SLO.
Meanwhile, the remaining training resources must be distributed across
multiple tasks. Treating all tasks uniformly under such workload changes
can be inefficient: resources may continue to be allocated to tasks with low
current inference demand while a highly requested task receives insufficient training
resources. Therefore, a joint scheduling mechanism should address two
coupled decisions: \emph{how much computation should be reserved for
inference to satisfy the SLO, and how should the remaining training capacity be distributed among tasks according to their current demand?}

Several recent works have investigated the coexistence of training and
inference on resource-constrained edge platforms. EdgeOL, for example,
studies efficient in-situ online learning on edge devices under limited
resources~\cite{li2024edgeol}. Other studies consider federated learning
systems that simultaneously provide model inference services and optimize
training and inference resource allocation across distributed
clients~\cite{han2024joint}. These studies demonstrate the importance of
coordinating learning and serving workloads. However, existing approaches
primarily focus on resource allocation, participation, or model-update
decisions, and do not explicitly connect \emph{bursty task-specific
inference demand} with the allocation of training effort across multiple
tasks.

In this paper, we study a multitask federated learning
system in which online inference and local training share limited
computation resources. Our key insight is that inference demand should
influence not only the amount of computation allocated to inference, but
also how the remaining capacity is distributed across training tasks.
A burst in task-specific inference requests should therefore trigger two coordinated
responses: sufficient resources should first be reserved for inference to
control service delay and maintain SLOs, while the remaining training capacity should
temporarily prioritize the high-demand task. By 
advancing the training progress of the high-demand task, 
the system can improve its model accuracy when that task is actively serving inference requests,
while maintaining
similar overall accuracy after training converges.

To realize this idea, we propose an \emph{SLO-aware joint inference and
training framework}, \ours, with queue-aware scheduling inspired by
Drift-Plus-Penalty (DPP) control~\cite{neely2010stochastic}.
The framework operates in two stages. First, an SLO-aware resource allocator
uses the current inference backlog and profiled device service capacity to
determine the minimum ratio of computation resources required for inference, leaving the
remaining capacity for local training. Second, a queue-aware DPP-inspired
scheduler uses task-specific pre-service queues as demand signals and
dynamically distributes the available training capacity among tasks. The
resulting allocation is then mapped to adaptive task loss weights, therefore coupling real-time inference demand with local multitask optimization.

This design provides a closed-loop interaction between inference serving
and federated training. When inference demand increases, the system
automatically reserves additional resources for request processing; when
the demand becomes concentrated on particular tasks, their growing queue states increase
their training priority. As demand changes over time, the scheduler dynamically adapts
both the inference-training resource split and the task-level training
allocation, while preserving shared representation learning through the
multitask model.
We further analyze the proposed framework theoretically. Specifically, we characterize
the minimum inference resource allocation required to satisfy the SLO, establish
inference queue stability under a supportable workload,
analyze the properties of the queue-aware training allocation, and characterize the convergence behavior of federated multitask training under
dynamic task weights.

We evaluate the proposed framework using a multitask vision workload
consisting of Stanford Cars~\cite{krause2013cars} and Oxford Flowers 102~\cite{nilsback2008flowers} with a shared ResNet
backbone and task-specific heads. The experiments emulate time-varying and
bursty task-specific inference arrivals while federated training proceeds
concurrently. The results demonstrate that {\ours} dynamically adapts
the inference-training resource split to workload changes and reallocates
training capacity toward high-demand tasks, improving model accuracy for actively requested tasks while maintaining overall multitask learning performance.

The main contributions of this paper are summarized as follows:

\begin{itemize}
    \item We formulate the coexistence of online inference and multitask
    federated training as a joint scheduling problem under time-varying and
    bursty task-specific inference demand.

    \item We develop an SLO-aware resource allocation mechanism that uses
    inference queue states and profiled service capacity to determine the
    inference-training computation split while preserving the maximum
    feasible capacity for local training.

    \item We design a queue-aware DPP-inspired training scheduler that
    distributes the remaining training capacity across tasks according to
    task-specific inference demand and translates the resulting allocation
    into dynamic multitask loss weights.

    \item We provide theoretical analysis of SLO-aware resource allocation,
    inference queue stability, and 
    federated multitask convergence under dynamic task weights, and evaluate the framework on a multitask federated
    learning testbed under bursty inference workloads.
\end{itemize}
\section{Related Work}
\label{sec:related}

\subsection{Federated Multitask Learning}

Federated learning (FL) enables distributed clients to collaboratively
train models while keeping raw data local. FedAvg established the standard
framework in which clients perform local optimization and the server
aggregates their model updates into a global model~\cite{mcmahan2017fedavg}. Subsequent studies have extended FL to
heterogeneous edge environments and multiple learning tasks.

Federated multitask learning (FMTL) considers scenarios in which multiple
related tasks are learned collaboratively while exploiting shared
representations or cross-task knowledge. Existing studies have investigated
multitask learning under communication and computation constraints.
Ma \emph{et al.} develop dynamic user and task scheduling for FMTL over
wireless networks~\cite{ma2023communication}, while Zhuang \emph{et al.}
propose MAS to efficiently coordinate multiple simultaneous federated
learning tasks~\cite{zhuang2023mas}.

These approaches primarily coordinate multitask training according to
training dynamics or system resource constraints. In contrast, our work
considers a different source of task heterogeneity:
\emph{time-varying inference demand}. 
We use online task-specific inference demand to determine when each task
should receive greater training emphasis. Consequently, a task experiencing high inference
demand can receive more training attention in earlier rounds, advancing its learning
progress and improving model accuracy when the task is actively requested, without necessarily
changing the final converged accuracy.



\subsection{Joint Training and Inference on Edge Devices}

Training and inference increasingly coexist on resource-constrained edge
platforms. Unlike the conventional train-then-deploy paradigm, continuously
learning systems must update their models while simultaneously serving
online inference requests. Since the two workloads compete for the same
computation resources, increasing training activity may reduce inference
service capacity, whereas prioritizing inference can slow model
improvement.

Recent studies have investigated this training-inference coupling.
For example, \textit{EdgeOL} develops an in-situ online learning framework
for continuously adapting models on resource-constrained edge devices while
maintaining inference performance~\cite{li2024edgeol}. Han \emph{et al.}
study federated learning while providing model-as-a-service and jointly
optimize learning participation, inference service, and computation and
communication resource allocation~\cite{han2024joint}. These studies
demonstrate the importance of coordinating model learning with online
inference services.

Our work differs from prior approaches in two important aspects. First, we consider a
\emph{multitask} model in which inference requests are associated with
different tasks and may exhibit distinct, time-varying demand patterns.
Second, inference demand affects not only the computation allocation between
inference and training, but also how the remaining training 
capacity is distributed across tasks. Specifically, the aggregate inference backlog
determines the SLO-aware inference-training resource split, while
task-specific queue states determine the relative training emphasis among
tasks. This two-level scheduling mechanism establishes a direct connection between online inference demand and
the temporal progression of multitask federated training.

\subsection{Queue-Aware Online Resource and Task Scheduling}

Queue-aware scheduling has been extensively studied for online resource
allocation in communication and computing systems. Lyapunov optimization~\cite{neely2010stochastic}
provides a general framework for making online decisions under stochastic
arrivals without requiring prior knowledge of future workloads.
In particular, Drift-Plus-Penalty (DPP) methods balance queue stability
against application-specific performance objectives and have been widely
used for dynamic resource allocation~\cite{neely2010stochastic}.

These principles are particularly relevant to bursty workloads because
queue states provide an online measure of accumulated demand. A growing
queue indicates that the current service or resource allocation is
insufficient relative to recent arrivals, allowing an online scheduler to
adapt without predicting future requests. Existing queue-aware approaches~\cite{han2024dynamic,bi2021lyapunov,dai2024lyapunov},
however, generally use queue information to control resources that
\emph{directly serve} the corresponding workload, such as computation,
communication, or service capacity.

Our framework uses queue information in two complementary roles. 
First, the aggregate
pre-service inference backlog is used to determine the minimum
inference resource ratio required to meet the target SLO. Second, the task-specific
pre-service queues serve as demand signals for a DPP-inspired
training scheduler. Rather than directly allocating training resources to serve these queues, the scheduler uses them to determine
the optimal distribution of the available training capacity
across tasks, which is subsequently mapped to dynamic multitask training
weights. Hence, inference demand influences both
\emph{how much} computation remains available for training and
\emph{which tasks} receive greater training emphasis.

\subsection{Summary and Distinction}

Existing FMTL studies primarily schedule training based on  learning
dynamics or system resource conditions, while joint training-inference
systems primarily determine how resources should be divided between
learning and serving workloads. In contrast, our framework connects these
two decisions through online, task-specific inference demand.
Specifically, bursty inference demand first adjusts the SLO-aware
inference-training resource split and then determines the relative training
priority of individual tasks. This allows the learning progress of a
high-demand task to be advanced to earlier communication rounds, enabling
the model to provide higher accuracy during periods when that task is
frequently requested, while preserving the long-term performance of the
multitask model.
\section{SLO-Aware Joint Inference and Training Framework}
\label{sec:method}

We consider a federated multitask learning system in which
online inference and local model training coexist on resource-constrained
edge clients. Since these workloads share the same computation resources,
increasing the resources allocated to inference reduces those available for
model training, and vice versa.
To address this coupling, we develop an SLO-aware, demand-driven multitask federated learning framework ({\ours}). Figure~\ref{fig:framework} illustrates the overall architecture of {\ours}, which consists of four main stages: 
bursty multitask inference demand, SLO-aware inference-training resource
allocation, queue-aware training scheduling, and federated multitask
learning.

Specifically, at each scheduling interval, newly arrived inference requests are first
incorporated into task-specific pre-service queues. A shared scheduler then
determines the minimum inference resource ratio required to satisfy a
predefined SLO and assigns the remaining
computation capacity to local training. The available training capacity is
subsequently distributed among learning tasks according to their current
inference demand. 
The resulting task allocations are mapped to dynamic task weights, which
determine the relative contribution of each task to the local multitask
training objective.

\begin{figure*}[t]
    \centering
    \includegraphics[width=\textwidth,height=0.46\textheight]{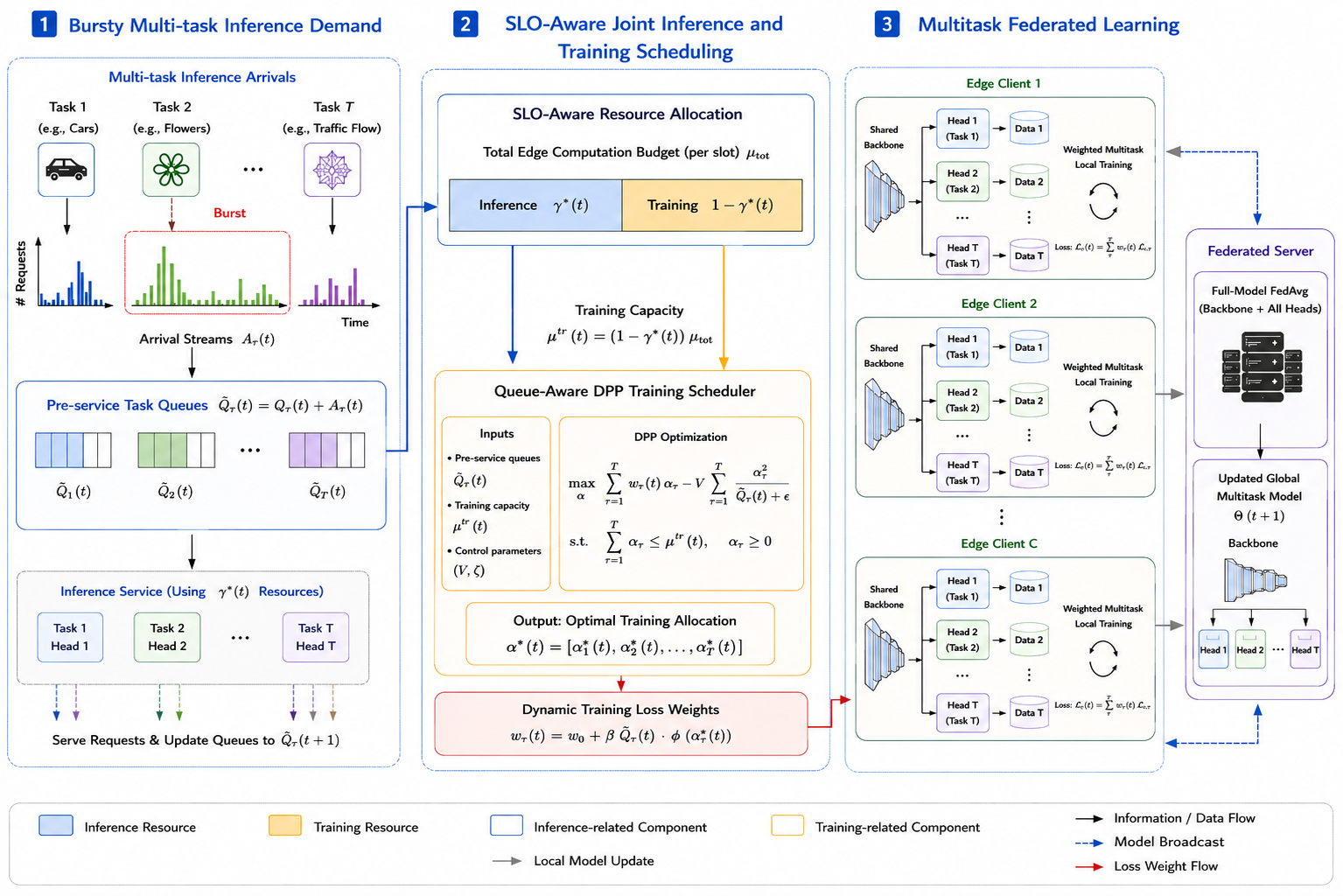}
    \caption{Overview of {\ours} for federated multitask learning. Bursty task-specific
    inference arrivals are incorporated into pre-service queues. The
    SLO-aware resource allocator first determines the inference resource
    ratio and assigns the remaining computation capacity to training.
    The shared queue-aware scheduler then distributes the available training
    capacity across tasks according to their current inference demand, and
    the resulting allocation determines the dynamic task weights used by
    selected clients for local multitask training. Finally, full-model
    FedAvg aggregates the shared backbone and all task-specific heads.}
    \label{fig:framework}
\end{figure*}

\subsection{System Model and Queue Dynamics}
\label{subsec:system}

Consider one parameter server and a set of edge clients
$\mathcal{C}=\{1,\ldots,C\}$. Each client $c\in\mathcal C$ maintains a
private local dataset $\mathcal D_c$ and participates in training a common
set of tasks $\mathcal T=\{1,\ldots,T\}$. Raw training samples remain on
the clients, while model parameters are exchanged with the server during
federated aggregation.

Each client maintains the same multitask model architecture
$\Theta=\{\Theta_b,\Theta_1,\ldots,\Theta_T\}$, where $\Theta_b$ denotes
the shared feature extractor and $\Theta_\tau$ denotes the task-specific
prediction head for task $\tau$. For an input $\mathbf{x}$, task $\tau$
produces
$f_\tau(\mathbf{x})
=h_\tau(g(\mathbf{x};\Theta_b),\Theta_\tau)$.

\textbf{Shared inference workload.}
We model inference demand at the system level because the scheduler makes
one shared resource-allocation decision for each scheduling interval $t$. 
Let $A_\tau(t)$ denote the number of newly arrived inference requests for
task $\tau$ during interval $t$, and let $Q_\tau(t)$ denote the backlog
carried from the previous interval.

Since the current arrivals should immediately influence the scheduling
decision, we define the pre-service queue as
\begin{equation}
\widetilde Q_\tau(t)
=
Q_\tau(t)+A_\tau(t).
\label{eq:pre_queue}
\end{equation}
Then, the aggregate pre-service backlog is
$\widetilde Q(t)=
\sum_{\tau\in\mathcal T}\widetilde Q_\tau(t)$.

Let $S_\tau(t)$ denote the number of inference requests of task $\tau$
served during interval $t$. After service, the queue evolves according to
\begin{equation}
Q_\tau(t+1)
=
\left[
\widetilde Q_\tau(t)-S_\tau(t)
\right]^+,
\label{eq:queue}
\end{equation}
where $[x]^+\triangleq\max\{x,0\}$.

This arrival-first queue convention allows both the SLO-aware resource
allocator and the training scheduler to react to inference bursts within
the same scheduling interval.

\textbf{Computation resources.}
Let $\widetilde{\phi}$ denote the maximum instantaneous computation
capacity available to the shared inference-training scheduler, and let
$\gamma(t)\in[0,1]$ denote the fraction assigned to inference.
Accordingly,
$\phi^{\mathrm{inf}}(t)=\gamma(t)\widetilde{\phi}$ and
$\phi^{\mathrm{tr}}(t)=(1-\gamma(t))\widetilde{\phi}$, where $\phi^{\mathrm{inf}}$ and $\phi^{\mathrm{tr}}$ are the computation resources allocated to inference and training, respectively.

Let $\mu^{\max}$ denote the maximum inference throughput under full
inference allocation, obtained through offline hardware profiling in
requests per second. Assuming that the effective service capability
approximately scales with the allocated computation ratio, the inference
service rate is
\begin{equation}
\mu^{\mathrm{inf}}(t)
=
\gamma(t)\mu^{\max}.
\label{eq:inf_rate}
\end{equation}


\subsection{SLO-Aware Inference--Training Resource Allocation}
\label{subsec:slo_allocation}

We next determine the shared computation ratio assigned to inference and
training. Motivated by Little's Law~\cite{little1961proof}, the time required to clear the current pre-service backlog
under the allocated inference capacity as
\begin{equation}
D(t)
=
\frac{\widetilde Q(t)}
{\mu^{\mathrm{inf}}(t)}
=
\frac{\widetilde Q(t)}
{\gamma(t)\mu^{\max}}.
\label{eq:delay}
\end{equation}

Let $D_{\mathrm{SLO}}$ denote the maximum acceptable inference delay.
The condition $D(t)\leq D_{\mathrm{SLO}}$ requires
\begin{equation}
\gamma(t)\geq
\frac{\widetilde Q(t)}{D_{\mathrm{SLO}}\mu^{\max}}.
\label{eq:SLO}
\end{equation}
We therefore define the workload-dependent required inference ratio as
\begin{equation}
\gamma^{\mathrm{req}}(t)
=
\frac{\widetilde Q(t)}
{D_{\mathrm{SLO}}\mu^{\max}}.
\label{eq:gamma_ratio}
\end{equation}

The scheduler assigns only the inference resources required by the current
workload while preserving the remaining computation for local training.
\begin{equation}
\gamma^{*}(t)
=
\left[
\frac{\widetilde Q(t)}
{D_{\mathrm{SLO}}\mu^{\max}}
\right]_{\gamma_{\min}}^{\gamma_{\max}},
\label{eq:gamma_opt}
\end{equation}
where
$[x]_a^b\triangleq\min\{b,\max\{a,x\}\}$.

Equivalently, using 
the maximum full-resource service capacity $M^{\max}=T_s\mu^{\max}$  during a scheduling interval of duration $T_s$, we have: 
\begin{equation}
\gamma^{\mathrm{req}}(t)
=
\frac{\widetilde Q(t)T_s}
{D_{\mathrm{SLO}}M^{\max}}.
\label{eq:splitratio_max}
\end{equation}
The two forms are identical; $\mu^{\max}$ is measured in requests per
second, whereas $M^{\max}$ is measured in requests per scheduling
interval.

If $\gamma^{\mathrm{req}}(t)>\gamma_{\max}$, the workload exceeds the
maximum inference allocation permitted by the resource constraint.
The scheduler therefore saturates at $\gamma_{\max}$, and the target SLO
is temporarily infeasible.

The remaining training ratio is
$\gamma^{\mathrm{tr}}(t)=1-\gamma^{*}(t)$.
Let $\mu^{\mathrm{tr,max}}$ denote the maximum effective training capacity
when all available computation is assigned to training. The training
capacity available after the inference decision is therefore
\begin{equation}
\mu^{\mathrm{tr}}(t)
=
\big(1-\gamma^{*}(t)\big)
\mu^{\mathrm{tr,max}}.
\label{eq:training_capacity}
\end{equation}

Thus, the first-stage decision determines the total amount of computation
available for training, while the second-stage scheduler determines how
this capacity is distributed among tasks.

\subsection{Queue-Aware Training Scheduling}
\label{subsec:dpp}

Now, we determine how the available training capacity
$\mu^{\mathrm{tr}}(t)$ is distributed among the tasks.

Let $\alpha_\tau(t)$ denote the fraction of the available training capacity
assigned to task $\tau$, subject to
$\sum_{\tau\in\mathcal T}\alpha_\tau(t)=1$ and
$\alpha_{\min}\leq\alpha_\tau(t)\leq\alpha_{\max}$.
The effective task-level training allocation is therefore
\begin{equation}
\mu_\tau^{\mathrm{tr}}(t)
=
\mu^{\mathrm{tr}}(t)\alpha_\tau(t).
\label{eq:taskleveltraining}
\end{equation}

Notice that $\gamma(t)$ and $\alpha_\tau(t)$ represent two distinct
scheduling decisions. The first-stage variable $\gamma(t)$ divides the
shared computation capacity between inference and training, whereas the
second-stage variable $\alpha_\tau(t)$ distributes only the remaining
training capacity among tasks.

\textbf{Queue-aware training utility.}
To prioritize tasks experiencing higher inference demand, we define
\begin{equation}
P(t)
=
\sum_{\tau\in\mathcal T}
\widetilde Q_\tau(t)
\log\!\left(
\mu^{\mathrm{tr}}(t)\alpha_\tau(t)+\epsilon_0
\right),
\label{eq:priorituzetask}
\end{equation}
where $\epsilon_0>0$ avoids the logarithmic singularity. A larger
$\widetilde Q_\tau(t)$ increases the scheduling importance of task $\tau$,
whereas the logarithmic utility introduces diminishing returns and
discourages a single task from monopolizing the available training
capacity.

\textbf{Temporal regularization.}
Responding directly to instantaneous demand can cause large allocation
changes across consecutive scheduling intervals. We therefore introduce
the switching cost
\begin{equation}
G(t)
=
\zeta
\sum_{\tau\in\mathcal T}
\big(
\alpha_\tau(t)-\alpha_\tau(t-1)
\big)^2,
\label{eq:switchcost}
\end{equation}
where $\zeta>0$ controls the intrinsic scale of temporal regularization.

Inspired by the Drift-Plus-Penalty (DPP) principle~\cite{neely2010stochastic}, we balance the workload-dependent training
utility against temporal allocation changes. The shared scheduler solves
\begin{align}
\boldsymbol{\alpha}^{*}(t)
=
\arg\min_{\boldsymbol{\alpha}(t)}
\quad&
V\zeta
\sum_{\tau\in\mathcal T}
\big(
\alpha_\tau(t)-\alpha_\tau(t-1)
\big)^2
\nonumber\\
&-
\sum_{\tau\in\mathcal T}
\widetilde Q_\tau(t)
\log\!\left(
\mu^{\mathrm{tr}}(t)\alpha_\tau(t)+\epsilon_0
\right)
\nonumber\\
\text{s.t.}\quad&
\sum_{\tau\in\mathcal T}\alpha_\tau(t)=1,
\nonumber\\
&
\alpha_{\min}
\leq
\alpha_\tau(t)
\leq
\alpha_{\max},
\quad\forall\tau.
\label{eq:dpp_opt}
\end{align}

The queue-dependent term makes the scheduler responsive to the current
inference workload. In particular, increasing
$\widetilde Q_\tau(t)$ strengthens the incentive to allocate training
capacity to task $\tau$. In contrast, the switching term discourages large
changes from the previous allocation. Therefore, the scheduler changes the
training distribution only when the workload-dependent utility gain is
sufficient to compensate for the temporal regularization cost.

The parameters $V$ and $\zeta$ have distinct roles. The coefficient
$\zeta$ sets the numerical scale of the switching cost, whereas $V>0$
controls the responsiveness--smoothness tradeoff. A smaller $V$ allows
the allocation to respond more aggressively to current inference demand,
while a larger $V$ favors smoother allocations across consecutive
intervals.

We emphasize that the inference queues are not directly served by
$\alpha_\tau(t)$. Instead, they provide workload-state information for the
training scheduler. Hence, (\ref{eq:dpp_opt}) is a DPP-inspired
queue-aware training allocation objective rather than the exact minimizer
of the inference queue's Lyapunov drift.

\subsection{Dynamic Multitask Training and Federated Aggregation}
\label{subsec:training}

The shared task allocation is incorporated into local multitask training
through dynamic task weights. Let $w_0>0$ denote the base task weight and
let $\beta\geq0$ control the strength of demand-aware prioritization.
For task $\tau$, we define
\begin{equation}
w_\tau(t)
=
w_0\big(1+\beta\alpha_\tau^{*}(t)\big),
\label{eq:weight-ratio}
\end{equation}
where $w_\tau(t)$ denotes the dynamic training weight assigned to task
$\tau$ at scheduling interval $t$, and $\alpha_\tau(t)^{*}$ is the optimal training emphasis determined by the shared queue-aware scheduler. A larger
$w_\tau(t)$ increases the contribution of task $\tau$ to the local
multitask training objective, thereby directing more training emphasis
toward tasks with higher current demand.

The normalized task weight is defined as:
\begin{equation}
\qquad
\omega_\tau(t)
=
\frac{w_\tau(t)}
{\sum_{j\in\mathcal T}w_j(t)}.
\label{eq:weight-ratio-normalized}
\end{equation}
Here, $\omega_\tau(t)$ represents the normalized training importance of
task $\tau$ relative to all tasks, satisfying
$\omega_\tau(t)\geq0$ and
$\sum_{\tau\in\mathcal T}\omega_\tau(t)=1$.
%
These normalized training weights 
$\{\omega_\tau(t)\}_{\tau\in\mathcal T}$ are then used by all clients selected
for the corresponding federated round. Consequently, inference demand
changes the relative training emphasis of the federation without requiring
independent per-client scheduling states.

Let $L_{c,\tau}(\Theta)$ denote the task-specific training loss at client
$c$. The local multitask objective is
\begin{equation}
L_c(\Theta;t)
=
\sum_{\tau\in\mathcal T}
\omega_\tau(t)L_{c,\tau}(\Theta).
\label{eq:local_objective}
\end{equation}

Client $c$ performs local optimization using the scheduler-weighted
multitask objective. For SGD, a local update can be written as
\begin{equation}
\Theta_c
\leftarrow
\Theta_c
-
\eta\nabla L_c(\Theta_c;t),
\label{eq:localupdate_c}
\end{equation}
where $\eta>0$ is the model learning rate and is unrelated to the
scheduler regularization coefficient $\zeta$.

After local training, each 
client uploads its complete multitask
model, including the shared backbone and all task-specific heads.
Let $\mathcal S_r$ denote the clients selected in communication round $r$.
The server performs full-model FedAvg:
\begin{equation}
\Theta^{r+1}
=
\sum_{c\in\mathcal S_r}
\frac{N_c}
{\sum_{k\in\mathcal S_r}N_k}
\Theta_c^{r+1},
\label{eq:fedavg}
\end{equation}
where $N_c$ denotes the number of local training samples at client $c$.
Thus, both the shared backbone and all task-specific heads are aggregated
into one global multitask model.

\subsection{Overall Scheduling and Federated Training Algorithm}
\label{subsec:algorithm}

Algorithm~\ref{alg:framework} summarizes the overall workflow of the
proposed framework. At each scheduling interval, the shared scheduler first
incorporates newly arrived inference requests into the task-specific
pre-service queues. It then computes the SLO-aware inference allocation
$\gamma^{*}(t)$ and determines the remaining training capacity. Based on
the task-specific pre-service queues, the scheduler solves the queue-aware
training allocation problem to obtain
$\boldsymbol{\alpha}^{*}(t)$ and converts the resulting task allocations
into dynamic task weights. The same scheduler-generated weights are then
used by all clients selected in the corresponding federated round. After
local multitask optimization, the server performs full-model FedAvg over
the shared backbone and all task-specific heads.

\begin{algorithm}[t]
\caption{Shared SLO-Aware Multitask FL Scheduler}
\label{alg:framework}

\KwIn{
Global model $\Theta^0$,
task set $\mathcal T$,
client set $\mathcal C$
}

Initialize $Q_\tau(0)=0$ and feasible $\alpha_\tau(0)$;

\For{$r=0,\ldots,R-1$}{

Set $t\leftarrow r$;

\textbf{Queue update:}
observe $A_\tau(t)$ and compute
$\widetilde Q_\tau(t)=Q_\tau(t)+A_\tau(t)$,
$\forall\tau\in\mathcal T$;

\textbf{Resource allocation:}
compute the SLO-aware inference ratio
$\gamma^*(t)$ using~(\ref{eq:gamma_opt}) and obtain
$\mu^{\mathrm{tr}}(t)$;

Serve inference requests and update
$Q_\tau(t+1)
=
[\widetilde Q_\tau(t)-S_\tau(t)]^+$;

\textbf{Training allocation:}
solve~(\ref{eq:dpp_opt}) to obtain
$\boldsymbol{\alpha}^*(t)$;

Compute and normalize the dynamic task weights
$\{\omega_\tau(t)\}_{\tau\in\mathcal T}$;

Select participating clients $\mathcal S_r$ and broadcast
$\Theta^r$ and $\{\omega_\tau(t)\}$;

\ForEach{$c\in\mathcal S_r$ \textbf{in parallel}}{
Train $\Theta_c^r$ using the weighted multitask objective
in~(\ref{eq:local_objective});
}

Aggregate the complete multitask models using
full-model FedAvg in~(\ref{eq:fedavg}) to obtain $\Theta^{r+1}$;

}

\end{algorithm}

\section{Theoretical Analysis}
\label{sec:analysis}

We analyze the proposed framework from four aspects: SLO-aware inference
allocation, inference-queue stability, queue-aware training emphasis, and
federated multitask convergence under dynamic task weights.

\subsection{Minimum SLO-Aware Resource Allocation}
\label{subsec:slo_analysis}

Given the aggregate pre-service backlog
$\widetilde Q(t)=\sum_{\tau}\widetilde Q_\tau(t)$, we define the
backlog-clearing delay surrogate as
\begin{equation}
D(t)
=
\frac{\widetilde Q(t)}
{\gamma(t)\mu^{\max}}.
\label{eq:delay_estimate_analysis}
\end{equation}
This is the time required to serve the current backlog under the allocated
capacity assuming no new arrivals, rather than a per-request latency
guarantee.

Requiring
$D(t)\leq D_{\mathrm{SLO}}$
implies
\begin{equation}
\gamma(t)
\geq
\gamma^{\mathrm{req}}(t)
\triangleq
\frac{\widetilde Q(t)}
{D_{\mathrm{SLO}}\mu^{\max}}.
\label{eq:slo_lower_bound}
\end{equation}

\begin{theorem}
\label{thm:slo_optimal}
The SLO-aware allocation
\begin{equation}
\gamma^{*}(t)
=
\min\!\left\{
\gamma_{\max},
\max\!\left\{
\gamma_{\min},
\gamma^{\mathrm{req}}(t)
\right\}
\right\}
\label{eq:slo_optimal}
\end{equation}
uses the minimum feasible inference resource required to meet the
backlog-clearing delay target whenever feasible, thereby maximizing the
remaining computation capacity for training.
\end{theorem}

\begin{proof}
From (\ref{eq:slo_lower_bound}), any feasible allocation satisfies
$\gamma(t)\geq\gamma^{\mathrm{req}}(t)$. Since the remaining training ratio
is $1-\gamma(t)$, selecting the smallest feasible $\gamma(t)$ subject to
$\gamma_{\min}\leq\gamma(t)\leq\gamma_{\max}$ yields
(\ref{eq:slo_optimal}). If
$\gamma^{\mathrm{req}}(t)>\gamma_{\max}$, the target is infeasible and the
scheduler saturates at $\gamma_{\max}$.
\end{proof}

\subsection{Inference Queue Stability}
\label{subsec:queue_stability}

Let $Q(t)=\sum_\tau Q_\tau(t)$ and $A(t)=\sum_\tau A_\tau(t)$, such that
$\widetilde Q(t)=Q(t)+A(t)$. Given inference allocation $\gamma(t)$ and
scheduling interval length $T_s$, the actual service is
\begin{equation}
S(t)
=
\min\left\{
\widetilde Q(t),
\gamma(t)\mu^{\max}T_s
\right\}.
\label{eq:actual_service}
\end{equation}
Hence,
\begin{equation}
Q(t+1)
=
[Q(t)+A(t)-S(t)]^+.
\label{eq:aggregate_queue}
\end{equation}

Let $S_{\max}=\gamma_{\max}\mu^{\max}T_s$ denote the maximum service
capacity per interval. We assume
\[
\mathbb{E}[A^2(t)\mid Q(t)]\leq A_2<\infty,
\qquad
0\leq S(t)\leq S_{\max},
\]
and a supportable workload such that, for some $\epsilon>0$,
\begin{equation}
\mathbb{E}[A(t)\mid Q(t)]
\leq
S_{\max}-\epsilon.
\label{eq:supportable}
\end{equation}

Using the quadratic Lyapunov function
$L(Q(t))=\frac{1}{2}Q^2(t)$, the standard drift argument
~\cite{neely2010stochastic} gives
\begin{equation}
\Delta(t)
\leq
B+
Q(t)\mathbb{E}[A(t)-S(t)\mid Q(t)],
\label{eq:drift_bound}
\end{equation}
where $B=\frac{1}{2}(A_2+S_{\max}^2)<\infty$.

\begin{theorem}
\label{thm:queue_stability}
Under the above assumptions, the aggregate inference queue under the
SLO-aware allocation is strongly stable.
\end{theorem}

\begin{proof}
Define
\[
Q_0
=
\max\left\{
\gamma_{\max}D_{\mathrm{SLO}}\mu^{\max},
S_{\max}
\right\}.
\]
Since $\widetilde Q(t)=Q(t)+A(t)\geq Q(t)$, whenever
$Q(t)\geq Q_0$, we have
$\gamma^{\mathrm{req}}(t)\geq\gamma_{\max}$ and hence
$\gamma^{*}(t)=\gamma_{\max}$. Moreover,
$\widetilde Q(t)\geq S_{\max}$, so
(\ref{eq:actual_service}) gives $S(t)=S_{\max}$. Therefore,
\[
\mathbb{E}[A(t)-S(t)\mid Q(t)]
\leq-\epsilon,
\qquad Q(t)\geq Q_0.
\]
Substituting into (\ref{eq:drift_bound}) yields
\[
\Delta(t)\leq B-\epsilon Q(t)
\]
outside the bounded region $Q(t)<Q_0$. The drift within this bounded region
is finite and can be absorbed into a finite constant. By the standard
Lyapunov stability criterion~\cite{neely2010stochastic},
\begin{equation}
\limsup_{R\rightarrow\infty}
\frac{1}{R}
\sum_{t=0}^{R-1}
\mathbb{E}[Q(t)]
<\infty.
\end{equation}
\end{proof}

Thus, temporary bursts may exceed the backlog-clearing delay target, while
long-term queue stability is preserved whenever the workload is supportable.

\subsection{Queue-Aware Training Emphasis}
\label{subsec:dpp_analysis}

We next analyze the DPP-inspired queue-aware training emphasis problem
(\ref{eq:dpp_opt}), where inference queues provide demand signals and
$\alpha_\tau(t)$ represents relative training emphasis rather than physical
GPU allocation.

Suppressing the time index, the objective is
\begin{equation}
f(\boldsymbol{\alpha})
=
V\zeta
\sum_{\tau}
(\alpha_\tau-\alpha_\tau^{\mathrm{prev}})^2
-
\sum_{\tau}
\widetilde Q_\tau
\log(
\mu^{\mathrm{tr}}\alpha_\tau+\epsilon_0
).
\label{eq:dpp_analysis_obj}
\end{equation}

\begin{theorem}
\label{thm:dpp_convex}
For $V>0$, $\zeta>0$, and $\epsilon_0>0$, if the simplex and box constraints
define a nonempty feasible region, the training-emphasis problem is strongly
convex and admits a unique global optimum
$\boldsymbol{\alpha}^{*}(t)$.
\end{theorem}

\begin{proof}
The Hessian is diagonal, with
\begin{equation}
\frac{\partial^2 f}
{\partial\alpha_\tau^2}
=
2V\zeta+
\frac{
\widetilde Q_\tau(\mu^{\mathrm{tr}})^2
}{
(\mu^{\mathrm{tr}}\alpha_\tau+\epsilon_0)^2
}
\geq
2V\zeta>0.
\end{equation}
Hence, the objective is strongly convex, implying a unique global optimum
over the convex feasible set.
\end{proof}

Moreover,
\begin{equation}
\frac{\partial f}{\partial\alpha_\tau}
=
2V\zeta
(\alpha_\tau-\alpha_\tau^{\mathrm{prev}})
-
\frac{
\widetilde Q_\tau\mu^{\mathrm{tr}}
}{
\mu^{\mathrm{tr}}\alpha_\tau+\epsilon_0
}.
\label{eq:alpha_gradient}
\end{equation}
A larger $\widetilde Q_\tau$ therefore strengthens the incentive to increase
the training emphasis of task $\tau$, while the quadratic term suppresses
abrupt changes. Hence, $V\zeta$ controls the
responsiveness--smoothness tradeoff.

Equation~(\ref{eq:dpp_opt}) is DPP-inspired rather than the exact minimizer
of the inference queue's Lyapunov drift: the queues provide demand signals
for training, whereas inference service is controlled by $\gamma(t)$.

\subsection{Federated Multitask Convergence}
\label{subsec:convergence}

The optimal training emphasis is mapped to a normalized task weight as
\begin{equation}
\omega_\tau^r
=
\frac{1+\beta\alpha_\tau^{*,r}}
{|\mathcal T|+\beta},
\qquad
\sum_{\tau\in\mathcal T}\omega_\tau^r=1,
\label{eq:normalized_task_weight}
\end{equation}
where $\beta\geq0$ controls the strength of demand-aware task prioritization.
Since $\alpha_\tau^{*,r}\geq0$,
\begin{equation}
\omega_\tau^r
\geq
\frac{1}{|\mathcal T|+\beta}
>0.
\label{eq:weight_lower_bound}
\end{equation}
Thus, every task retains a nonzero contribution to the local training
objective. When $\beta=0$, all tasks receive equal weights, whereas
increasing $\beta$ makes the task weights more responsive to the scheduler
output $\alpha_\tau^{*,r}$.

The local and global round-dependent objectives are
\begin{equation}
F_c^r(\Theta)=\sum_{\tau}\omega_\tau^rF_{c,\tau}(\Theta),
\qquad
F^r(\Theta)=\sum_c p_cF_c^r(\Theta),
\end{equation}
where $p_c=N_c/\sum_jN_j$.

We assume standard $L$-smoothness, bounded stochastic-gradient variance
$\sigma^2$, and bounded client heterogeneity:
\begin{equation}
\sum_{c=1}^{C}
p_c
\left\|
\nabla F_c^r(\Theta)-\nabla F^r(\Theta)
\right\|^2
\leq
G^2.
\end{equation}
To capture scheduler-induced objective changes, assume
\begin{equation}
\left|
F^{r+1}(\Theta^{r+1})
-
F^r(\Theta^{r+1})
\right|
\leq
\delta_r,
\qquad
\bar{\delta}_R
=
\frac{1}{R}
\sum_{r=0}^{R-1}\delta_r.
\label{eq:objective_variation}
\end{equation}

\begin{theorem}
\label{thm:fl_convergence}
Suppose each participating client performs $E$ local SGD steps with a
sufficiently small learning rate $\eta$, followed by FedAvg aggregation.
Then
\begin{equation}
\begin{aligned}
\frac{1}{R}\sum_{r=0}^{R-1}
\mathbb{E}\|\nabla F^r(\Theta^r)\|^2
\leq\;&
\frac{2(F^0(\Theta^0)-F_{\inf})}
{\eta E R}
\\
&+\mathcal{O}(\eta\sigma^2)
+\mathcal{O}(\eta^2E^2G^2)
\\
&+\mathcal{O}\!\left(
\frac{\bar{\delta}_R}{\eta E}
\right),
\end{aligned}
\label{eq:convergence_bound}
\end{equation}
where $F_{\inf}$ is a common finite lower bound.
\end{theorem}

\begin{proof}[Proof sketch]
For a fixed round-$r$ objective, standard Local-SGD/FedAvg analysis
~\cite{stich2019local,karimireddy2020scaffold,glasgow2022sharp}
gives, for sufficiently small $\eta$,
\begin{equation}
\begin{aligned}
\mathbb{E}[F^r(\Theta^{r+1})]
\leq\;&
\mathbb{E}[F^r(\Theta^r)]
-c\eta E\,
\mathbb{E}\|\nabla F^r(\Theta^r)\|^2
\\
&+
C_1\eta^2E\sigma^2
+
C_2\eta^3E^3G^2,
\end{aligned}
\label{eq:round_descent}
\end{equation}
for constants $c,C_1,C_2>0$.

Since the scheduler changes the task weights between rounds,
(\ref{eq:objective_variation}) implies
\[
F^{r+1}(\Theta^{r+1})
\leq
F^r(\Theta^{r+1})+\delta_r.
\]
Combining the two inequalities and summing over rounds gives
\begin{equation}
\begin{aligned}
c\eta E
\sum_{r=0}^{R-1}
\mathbb{E}\|\nabla F^r(\Theta^r)\|^2
\leq\;&
F^0(\Theta^0)-F_{\inf}
\\
&+\mathcal{O}(R\eta^2E\sigma^2)
\\
&+\mathcal{O}(R\eta^3E^3G^2)
+R\bar{\delta}_R.
\end{aligned}
\end{equation}
Dividing by $c\eta ER$ and absorbing fixed constants into
$\mathcal{O}(\cdot)$ yields (\ref{eq:convergence_bound}).
\end{proof}

The bound contains the standard optimization, stochastic-gradient, and
client-heterogeneity terms, together with an additional term caused by
scheduler-induced objective variation. If
$\eta=\Theta(R^{-1/2})$, $E$ is fixed, and
$\bar{\delta}_R=o(R^{-1/2})$, this additional tracking term vanishes
asymptotically. This is a conditional result and does not imply that the
scheduler itself guarantees this decay rate.
\section{Evaluation Results}
\label{sec:experiment}

\subsection{Multitask FL Setup}

We evaluate our framework on a two-task vision workload consisting of
Stanford Cars~\cite{krause2013cars} and Oxford Flowers 102~\cite{nilsback2008flowers}, using a shared ResNet backbone with
task-specific classification heads. All experiments are conducted on an NVIDIA
Tesla V100-SXM2 GPU with 32\,GB of memory.

We use 100 federated clients and randomly select 30\% of them in each
communication round. Each selected client performs one local epoch with
200 training steps using a learning rate of $10^{-3}$. Full-model FedAvg is used to aggregate both the shared backbone and all the task-specific heads.

\subsection{Inference Workloads}

We consider both synthetic and real-trace-driven inference workloads. 
For the synthetic workload, we generate independent Poisson arrivals for Cars and
Flowers with a baseline rate of $3$ requests/s. To emulate bursty demand,
the arrival rate of the selected task is increased by a factor of $4$
during predefined burst intervals. Request timestamps are aggregated into
fixed-length windows and used as arrival inputs to the online scheduler.

For the real-trace-driven workload, we derive task arrivals from the
Alibaba Cluster Trace 2018~\cite{alibaba2018trace}. The trace contains
production workloads collected from approximately 4,000 machines over
eight days, including colocated online services and batch workloads.
Instances are divided into two task streams according to their normalized
CPU and memory demands, with high-demand instances assigned to Cars and the
remaining instances assigned to Flowers.
We then normalize the two streams to maximum arrival counts of 900 and 200
requests per slot, respectively.

Figure~\ref{fig:workloads} shows the two inference workloads used in our experiments.
As shown in Figure~\ref{fig:workloads}, the synthetic workload contains a controlled burst in Stanford
Cars while Oxford Flowers remains relatively stable. In contrast, the
Alibaba trace-derived workload exhibits multiple irregular demand peaks.
Together, these workloads allow us to evaluate \ours{} under both controlled
and realistic workload dynamics.

\subsection{Scheduling and Metrics}

We obtain the maximum inference service capacity from offline hardware profiling.
At each scheduling interval, the proposed scheduler first determines the
minimum inference resource ratio required to satisfy the target SLO and allocates the remaining capacity to training. It then uses task-specific inference
demand to determine the training allocation across tasks and the corresponding dynamic loss
weights.

We compare \ours{} with two scheduling baselines. \textbf{Round Robin}
distributes the available training opportunities evenly across tasks,
without considering task-specific inference demand.
\textbf{Burst-Aware FIFO} prioritizes training according to the current
inference demand, assigning training opportunities first to the task with
the dominant request workload. It therefore reacts aggressively to workload
bursts, but may provide limited training opportunities to lower-demand tasks.

We report per-task accuracy, overall multitask accuracy, inference and
training resource ratios, and queue backlog. We evaluate accuracy during burst
periods to determine whether demand-aware scheduling advances the learning
progress of high-demand tasks to earlier rounds while maintaining comparable
long-term model accuracy.

\subsection{Overall Performance}
\label{sec:overall}


\begin{figure}[t]
    \centering
    \begin{minipage}{0.49\columnwidth}
        \centering
        \includegraphics[
            width=\linewidth,
            height=0.14\textheight
        ]{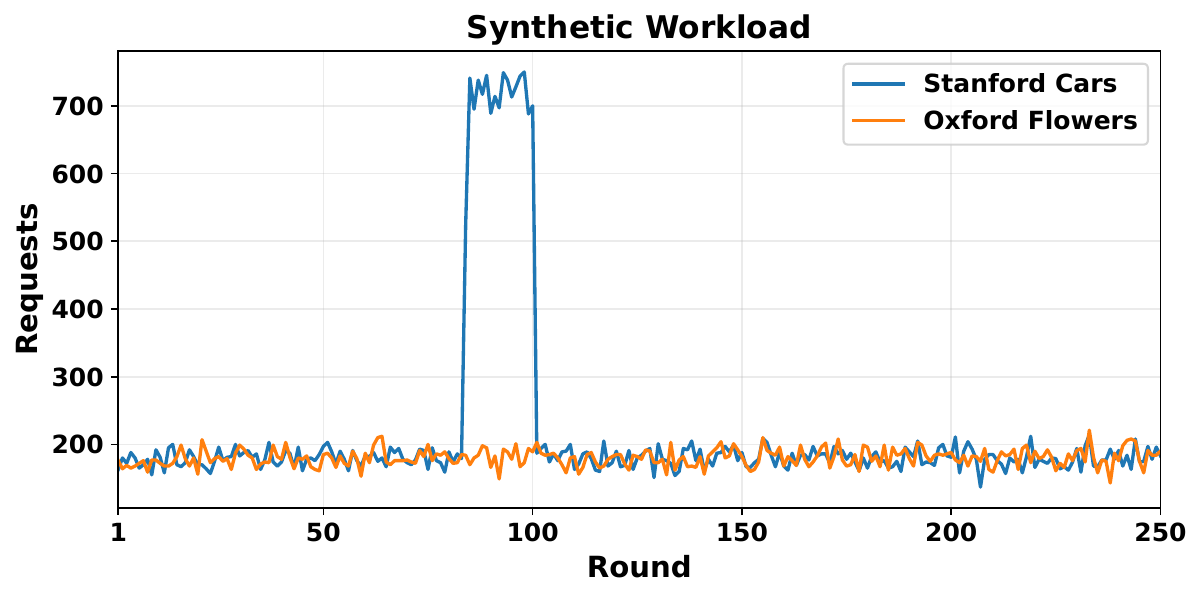}\\[-3mm]
        {\footnotesize (a) Synthetic workload}
    \end{minipage}
    \hfill
    \begin{minipage}{0.49\columnwidth}
        \centering
        \includegraphics[
            width=\linewidth,
            height=0.14\textheight
        ]{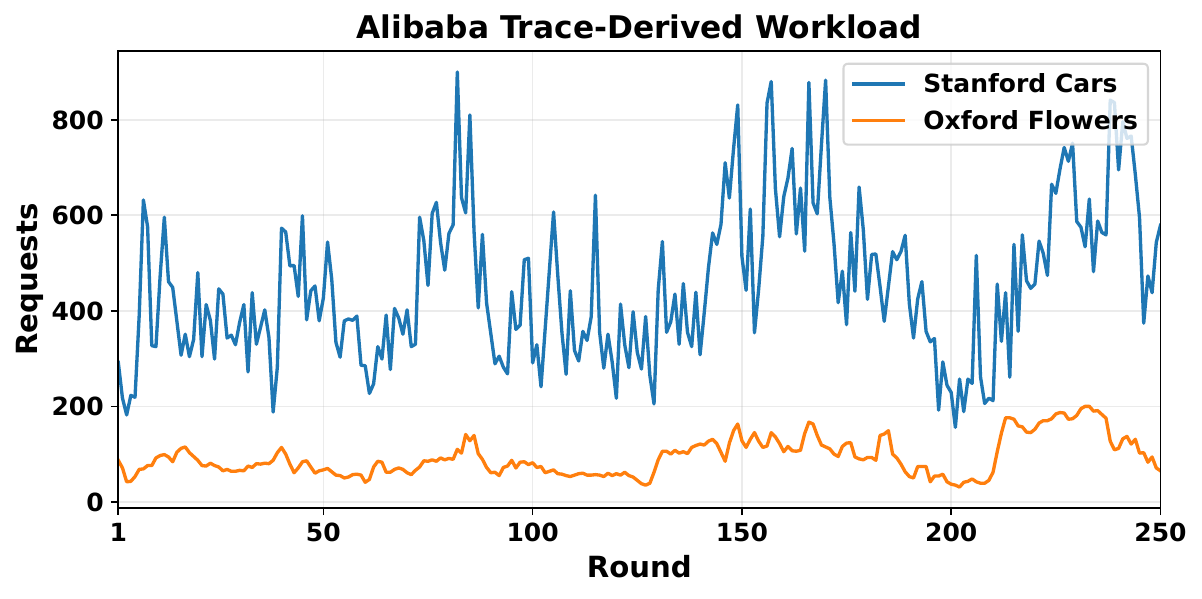}\\[-3mm]
        {\footnotesize (b) Alibaba workload}
    \end{minipage}

    \caption{Task-specific inference workloads: (a) synthetic and
    (b) real Alibaba cluster trace, used in our experiments.}
    \label{fig:workloads}
\end{figure}
Figure~\ref{fig:synthetic_acc} reports the inference accuracy under the
synthetic workload. During the Cars burst, \ours{} dynamically increases
the training priority of Stanford Cars. Averaged over the burst period,
\ours{} improves the Cars inference accuracy by 5.17\%
over Round Robin, with the instantaneous improvement reaching up to
15.9\% during the most demanding burst rounds.
These results demonstrate that demand-aware training enables the high-demand task
to reach higher accuracy earlier, precisely when its inference service demand
increases. This additional training emphasis on Stanford Cars temporarily reduces the
training capacity available to Oxford Flowers, resulting in a moderate
decrease in its inference accuracy during the same period.

Burst-Aware FIFO exhibits a more aggressive behavior. Since Cars dominates
the workload during the burst, it assigns nearly all available training
opportunities to Cars. Consequently, Stanford Cars achieves the highest inference
accuracy, but Oxford Flowers receives almost no training and therefore
maintains very low inference accuracy. This result exposes the limitation of
directly following the dominant workload: aggressively prioritizing the
high-demand task can effectively starve other tasks. In contrast,
\ours{} increases the training priority of Stanford Cars while preserving non-zero 
training capacity for Oxford Flowers, providing a more balanced response to bursty demand.

\begin{figure}[t]
    \centering
    \begin{minipage}{0.49\columnwidth}
        \centering
        \includegraphics[
            width=\linewidth,
            height=0.14\textheight
        ]{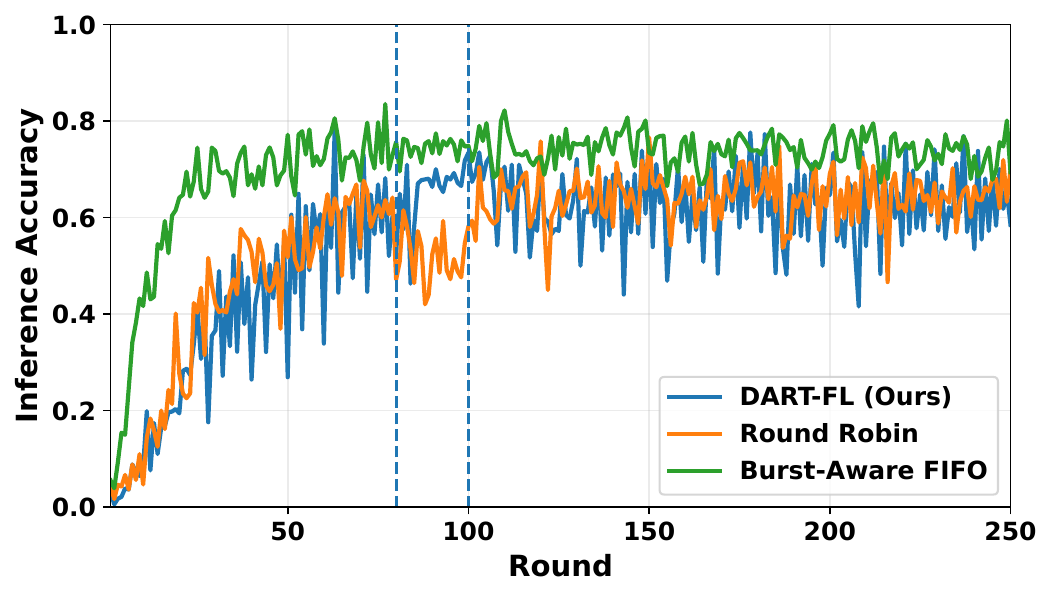}\\[-3mm]
        {\footnotesize (a) Stanford Cars}
    \end{minipage}
    \hfill
    \begin{minipage}{0.49\columnwidth}
        \centering
        \includegraphics[
            width=\linewidth,
            height=0.14\textheight
        ]{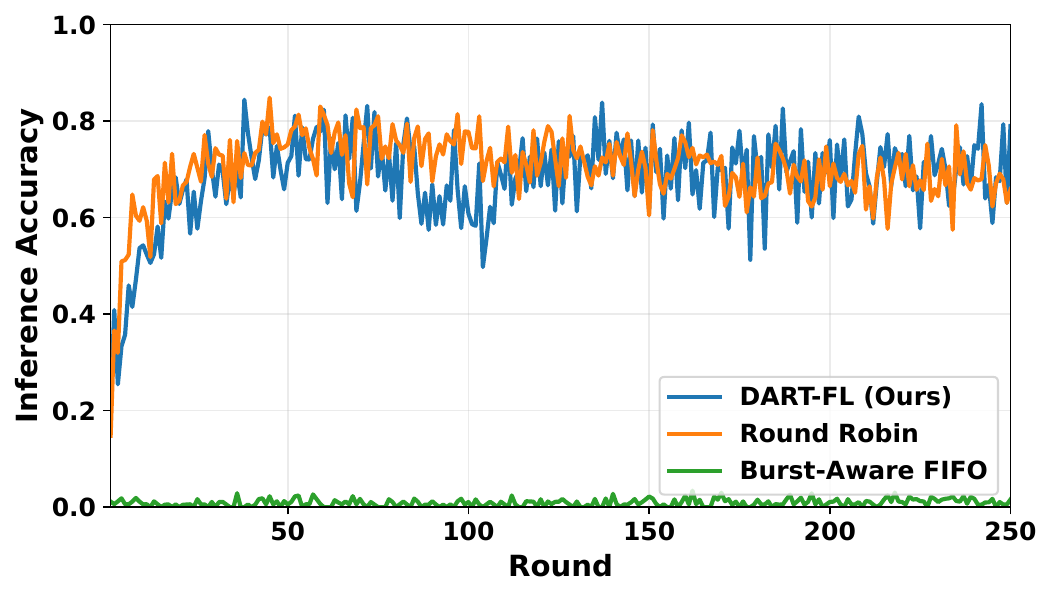}\\[-3mm]
        {\footnotesize (b) Oxford Flowers}
    \end{minipage}

    \caption{Inference accuracy under the synthetic workload. The dashed
    lines indicate the burst period in Stanford Cars.}
    \label{fig:synthetic_acc}
\end{figure}

Figure~\ref{fig:alibaba_acc} further evaluates the methods using the
Alibaba trace-derived workload. Unlike the controlled synthetic workload,
the Alibaba trace contains multiple irregular Cars demand peaks, highlighted
by the dashed regions in Fig.~\ref{fig:alibaba_acc}(a). We compute the
average inference accuracy over all rounds within each peak interval.
Compared with Round Robin, \ours{} improves the average Cars inference
accuracy from 0.618 to 0.654, from 0.650 to 0.727, and from 0.641 to
0.712 over the three peak periods, corresponding to relative improvements of 5.83\%, 11.85\%, and 11.08\%, 
respectively. These results show that
\ours{} consistently adapts training priorities to changes in task demand, even under the more
irregular workload dynamics of the real-world trace.

The increased emphasis on Stanford Cars is accompanied by a temporary reduction in
Oxford Flowers accuracy because less training capacity is allocated to the lower-demand
task during these burst periods. Nevertheless, \ours{} continues to reserve training
capacity for Flowers, preventing its learning progress from being completely stalled. Burst-Aware FIFO again represents a more aggressive
alternative: by heavily prioritizing the dominant Cars workload, it achieves
high Cars accuracy at the cost of providing very limited training
opportunities to Flowers. As a result, its Flowers inference accuracy
remains substantially low. \ours{} therefore achieves a better balance between
responding to high-demand tasks and maintaining the learning progress
of other tasks.

\begin{figure}[t]
    \centering
    \begin{minipage}{0.49\columnwidth}
        \centering
        \includegraphics[
            width=\linewidth,
            height=0.14\textheight
        ]{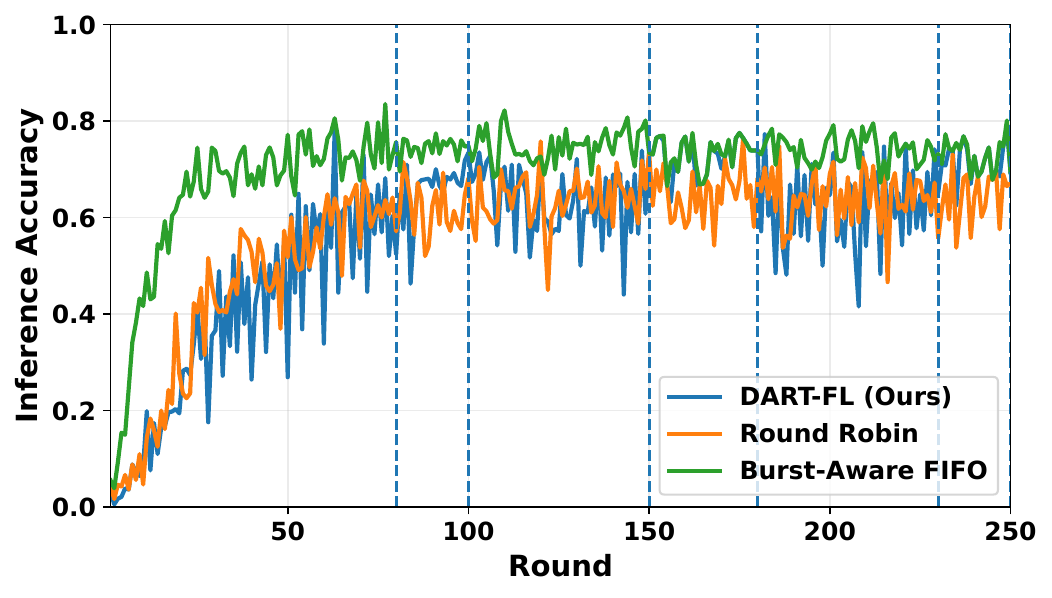}\\[-3mm]
        {\footnotesize (a) Stanford Cars}
    \end{minipage}
    \hfill
    \begin{minipage}{0.49\columnwidth}
        \centering
        \includegraphics[
            width=\linewidth,
            height=0.14\textheight
        ]{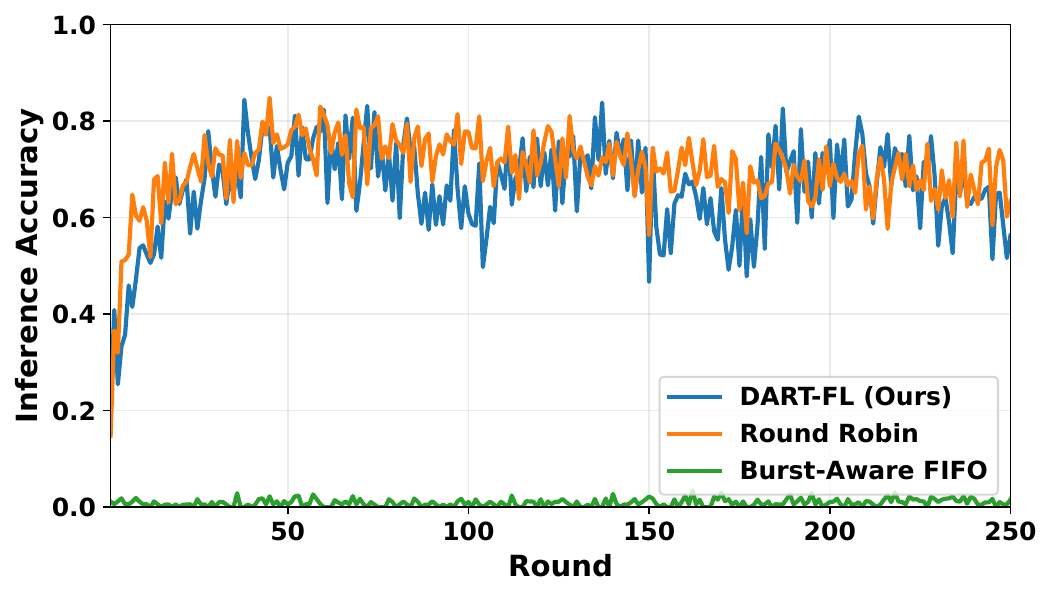}\\[-3mm]
        {\footnotesize (b) Oxford Flowers}
    \end{minipage}

    \caption{Inference accuracy under the Alibaba trace-derived workload.
    The dashed lines indicate the burst periods in Stanford Cars.}
    \label{fig:alibaba_acc}
\end{figure}

\subsection{Runtime Adaptation Behavior}
\label{sec:runtimeadapation}

To understand how these accuracy improvements are achieved,
Figure~\ref{fig:scheduling_behavior} illustrates the scheduling behavior of
\ours{} under the synthetic workload. As shown in Figure~\ref{fig:scheduling_behavior}(a), the inference
resource ratio increases from approximately 0.2 to 0.5 during the Cars
burst, while the training ratio decreases accordingly. The SLO-aware
resource allocator therefore reacts to the increased inference demand by
reserving more computation for request processing instead of maintaining a
fixed inference-training resource split.

A similar adaptive behavior is observed under the Alibaba workload in
Figure~\ref{fig:scheduling_behavior}(c). Since the Alibaba trace contains
multiple irregular demand peaks, the inference--training resource split
varies more frequently over time. During periods of increased inference
demand, \ours{} increases the inference resource ratio and correspondingly
reduces the resources assigned to training. When the demand decreases,
more computation is returned to model training. These results demonstrate
that \ours{} can dynamically adapt the inference--training resource split
under both controlled and real-world workload variations.

At the same time, \ours{} adapts how the remaining training capacity is
distributed across tasks. As shown in Figure~\ref{fig:scheduling_behavior}(b), the normalized
training loss weights are approximately balanced before the burst. When the
Cars burst occurs, the Stanford Cars weight increases from approximately
0.5 to 0.6, while the Oxford Flowers weight decreases from approximately
0.5 to 0.4. Thus, \ours{} performs adaptation at two levels: it first
increases the resources allocated to inference and then shifts the remaining
training effort toward the high-demand Cars task. This coordinated adaptation enables Cars to
improve its model accuracy during the burst and explains the corresponding inference
accuracy gains in Figure~\ref{fig:synthetic_acc}.

A similar behavior is observed under the Alibaba trace-derived workload.
As shown in Figure~\ref{fig:scheduling_behavior}(d), the normalized task
weights dynamically change with the time-varying demand. The Stanford Cars
weight remains generally higher when Cars demand dominates, while the Oxford
Flowers weight is reduced accordingly. Compared with the controlled
synthetic burst, these changes are more irregular because the Alibaba trace
contains multiple demand fluctuations. Together with the adaptive
inference-training resource split in
Figure~\ref{fig:scheduling_behavior}(c), these results show that \ours{} can
jointly adapt resource allocation and task-level training priorities under
real-world workload dynamics.

Importantly, the Oxford Flowers training weight remains non-zero throughout the
burst. This distinguishes \ours{} from Burst-Aware FIFO, which can devote
nearly all available training opportunities to the dominant Cars workload and
effectively stall the Flowers training. Hence, \ours{} responds to bursty
demand while continuing to preserve training progress for lower-demand tasks. 


\begin{figure}[t]
    \centering

    \begin{minipage}{0.49\columnwidth}
        \centering
        \includegraphics[
            width=\linewidth,
            height=0.14\textheight
        ]{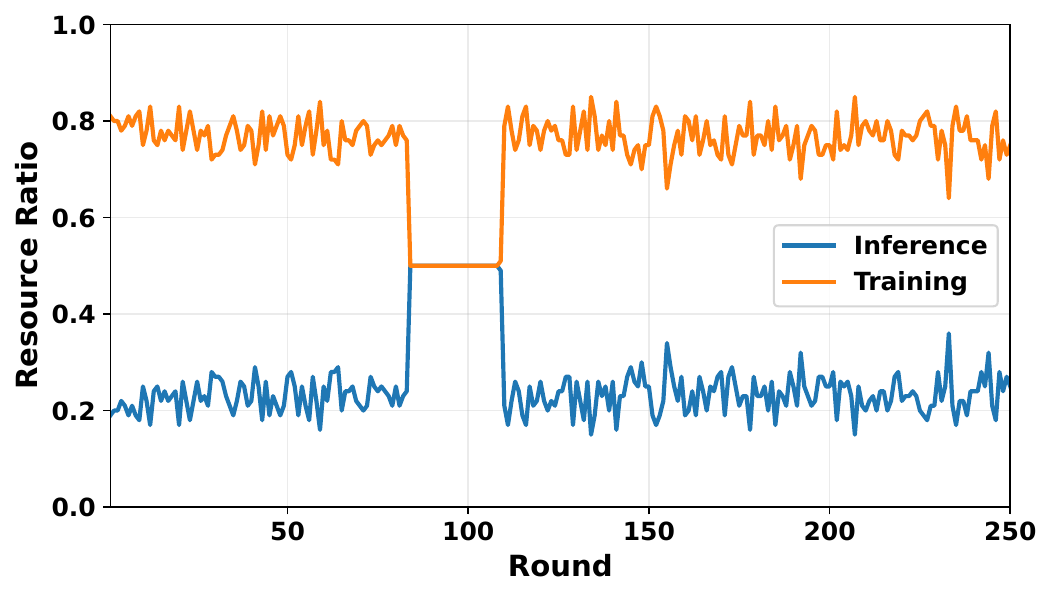}\\[-3mm]
        {\footnotesize (a) Synthetic: resource ratio}
    \end{minipage}
    \hfill
    \begin{minipage}{0.49\columnwidth}
        \centering
        \includegraphics[
            width=\linewidth,
            height=0.14\textheight
        ]{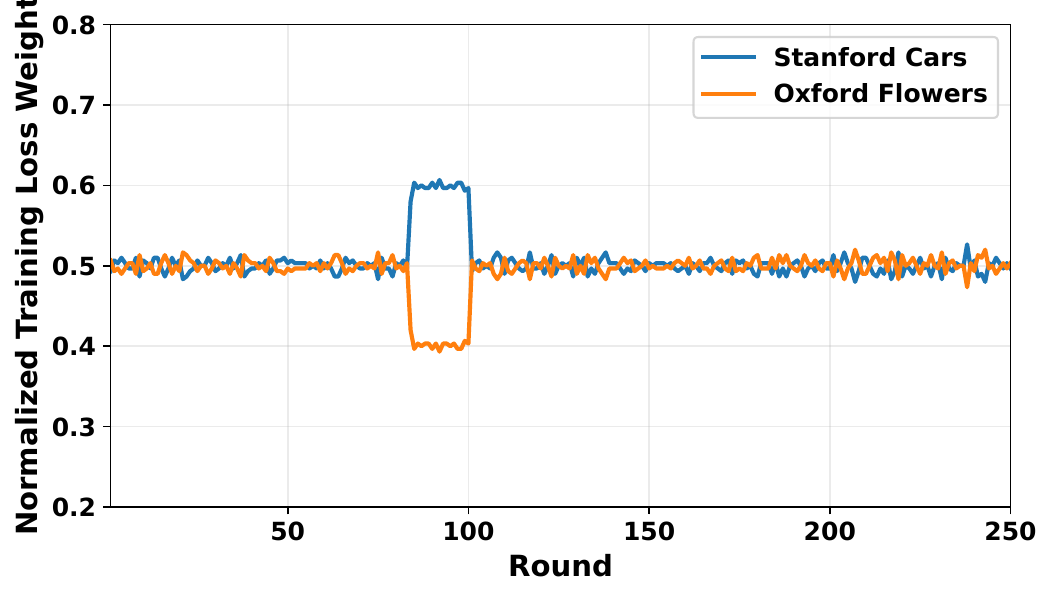}\\[-3mm]
        {\footnotesize (b) Synthetic: task weights}
    \end{minipage}

    \vspace{2mm}

    \begin{minipage}{0.49\columnwidth}
        \centering
        \includegraphics[
            width=\linewidth,
            height=0.14\textheight
        ]{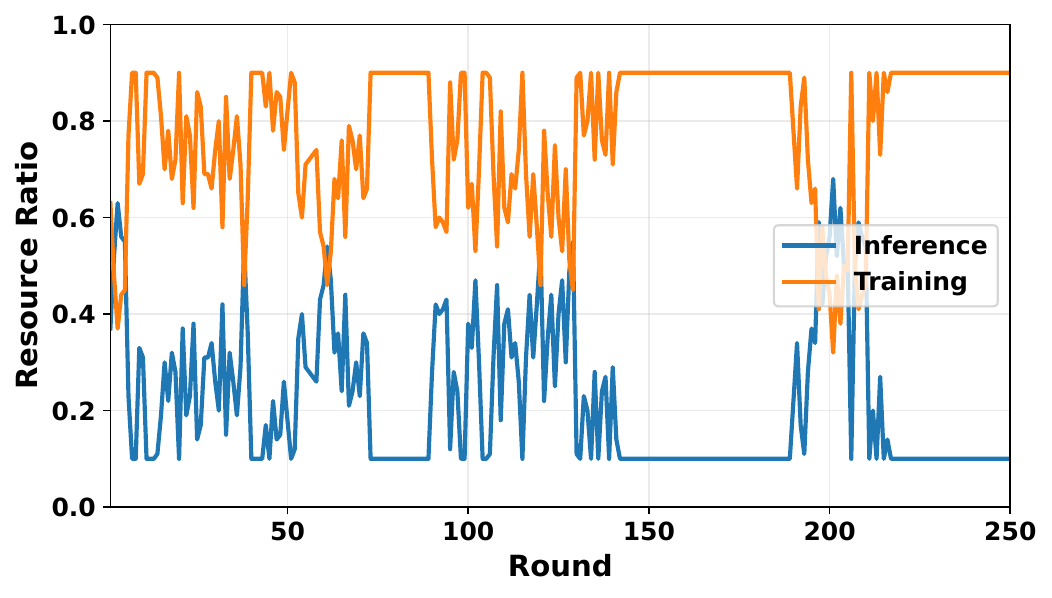}\\[-3mm]
        {\footnotesize (c) Alibaba: resource ratio}
    \end{minipage}
    \hfill
    \begin{minipage}{0.49\columnwidth}
        \centering
        \includegraphics[
            width=\linewidth,
            height=0.14\textheight
        ]{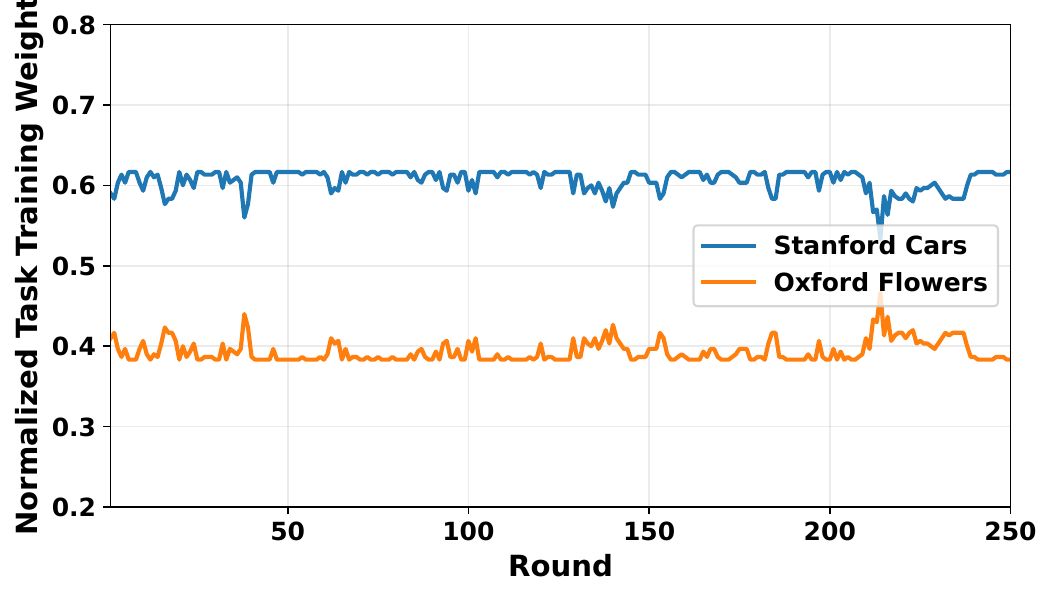}\\[-3mm]
        {\footnotesize (d) Alibaba: task weights}
    \end{minipage}

    \caption{Adaptive scheduling behavior of \ours{} under synthetic and
    Alibaba trace-derived inference workloads.}
    \label{fig:scheduling_behavior}
\end{figure}

Overall, the results demonstrate the trade-off between demand responsiveness
and multitask balance. Round Robin maintains balanced training but cannot
respond to task-specific demand changes, whereas Burst-Aware FIFO
aggressively follows the dominant workload and can starve lower-demand
tasks. \ours{} avoids these two extremes by jointly adapting the
inference-training resource split and the task-level training weights.
Consequently, it accelerates learning and improves model accuracy for high-demand tasks during burst periods while preserving appropriate training opportunities for lower-demand tasks, achieving a better balance between short-term demand responsiveness and sustained multitask learning.

\section{Conclusion}
\label{sec:conclusion}

In this paper, we presented \ours{}, a burst-aware scheduling framework for
multitask federated learning with concurrent inference and training. \ours{}
dynamically adjusts the inference--training resource split and uses a
queue-aware DPP-inspired scheduler to adapt task-level training weights according to
time-varying inference demand.

Experiments with synthetic and Alibaba trace-derived workloads show that
\ours{} improves the inference accuracy of high-demand tasks during burst
periods while preserving training opportunities for other tasks. Compared
with Round Robin and Burst-Aware FIFO, \ours{} achieves a better balance
between demand responsiveness and multitask learning. Future work will extend
the framework to larger-scale multitask systems with more heterogeneous
clients and dynamic workloads.

\bibliographystyle{IEEEtran}
\bibliography{cumtom}

\end{document}